\documentclass[conference,letterpaper]{IEEEtran}

\usepackage[utf8]{inputenc} 
\usepackage[T1]{fontenc}
\usepackage{url}
\usepackage{ifthen}
\usepackage{cite}
\usepackage[cmex10]{amsmath} 

\usepackage{bbm}
\usepackage{enumitem}

\usepackage{caption}
\usepackage{subcaption}
\usepackage{wrapfig}

\usepackage{xcolor}

\long\def\comment#1{}

\newfont{\bbb}{msbm10 scaled 700}

\newfont{\bb}{msbm10 scaled 1100}

\newcommand{\EE}{\mathbb{E}}

\newcommand{\RR}{\mathbb{R}}

\newcommand{\Dc}{\mathcal{D}}

\newcommand{\Nc}{\mathcal{N}}

\newcommand{\Sc}{\mathcal{S}}

\newcommand{\Xc}{\mathcal{X}}

\newcommand{\cs}{{\boldsymbol c}}

\newcommand{\us}{{\boldsymbol u}}

\newcommand{\vs}{{\boldsymbol v}}
\newcommand{\xs}{{\boldsymbol x}}
\newcommand{\ys}{{\boldsymbol y}}
\newcommand{\zs}{{\boldsymbol z}}

\renewcommand{\arg}{{\hbox{arg}}}
\newcommand{\var}{{\hbox{Var}}}

\usepackage{amsmath}
\usepackage{amssymb}
\usepackage{mathtools}
\usepackage{amsthm}

\theoremstyle{plain}
\newtheorem{theorem}{Theorem}

\newtheorem{lemma}{Lemma}

\newtheorem{corollary}{Corollary}

\theoremstyle{definition}
\newtheorem{definition}{Definition}

\theoremstyle{remark}
\newtheorem{remark}{Remark}

\newtheorem{example}{Example}

\begin{document}
\title{Data-Driven Estimation of the False Positive Rate 
of the Bayes Binary Classifier via Soft Labels} 



\author{%
  \IEEEauthorblockN{Minoh~Jeong\IEEEauthorrefmark{1},
                    Martina~Cardone\IEEEauthorrefmark{1},
                    and Alex~Dytso\IEEEauthorrefmark{2},
                    }
  \IEEEauthorblockA{\IEEEauthorrefmark{1}%
                   Department of ECE, University of Minnesota, Minneapolis, MN, USA,
                    \{jeong316, mcardone\}@umn.edu}
  \IEEEauthorblockA{\IEEEauthorrefmark{2}%
                    Qualcomm Flarion Technology, Inc.,
                    Bridgewater, NJ, USA,
                    odytso2@gmail.com}
}

\maketitle


\begin{abstract}
Classification is a fundamental task in many applications on which data-driven methods have shown outstanding performances. However, it is challenging to determine whether such methods have achieved the optimal performance. 
This is mainly because the best achievable performance is typically unknown and hence, effectively estimating it is of prime importance. 
In this paper, we consider binary classification problems and we propose an estimator for the false positive rate (FPR) of the Bayes classifier, that is, the optimal classifier with respect to accuracy, from a given dataset.
Our method utilizes soft labels, or real-valued labels, which are gaining significant traction thanks to their properties.
We thoroughly examine various theoretical properties of our estimator, including its consistency, unbiasedness, rate of convergence, and variance.
To enhance the versatility of our estimator beyond soft labels, we also consider noisy labels, which encompass binary labels. For noisy labels, we develop effective FPR estimators by leveraging a denoising technique and the Nadaraya-Watson estimator. 
Due to the symmetry of the problem, our results can be readily applied to estimate the false negative rate of the Bayes classifier.
\end{abstract}

\section{Introduction}

In the rapidly evolving landscape of data-driven decision-making, binary classifiers have emerged as fundamental tools in many real-world applications, such as medical diagnosis~\cite{mouliou2021false,hammer2022factors}, fraud detection~\cite{bolton2002statistical,abdallah2016fraud}, manufacturing~\cite{fathy2020learning}, facial recognition~\cite{garcia2004convolutional}, and cybersecurity~\cite{larriva2020evaluation}.
The effectiveness of classifiers, however, is not solely determined by their accuracy; it is, in fact, fundamentally influenced by their ability to minimize erroneous critical decisions, which can be measured by false positives (FP) and false negatives (FN). FP is the number of decisions that incorrectly predict a positive outcome and FN is the number of decisions that fail to identify a positive outcome. These play pivotal roles in assessing classifiers when the costs of incorrect decisions for positive or negative outcomes are different. Besides, FP and FN have connections to other useful and widely used metrics (e.g., precision, recall, and sensitivity)~\cite{davis2006relationship,hastie2009elements}, showing their significance.

In this work, we seek to characterize the false positive rate (FPR) and the false negative rate (FNR) for binary classification problems when the Bayes classifier, which is the theoretical optimal classifier with respect to accuracy, is used. Specifically, we formulate an estimation problem, in which we estimate the FPR and the FNR of the Bayes classifier from a given dataset. We consider two types of labels in the dataset, namely soft labels and noisy labels, which we mainly leverage to estimate the FPR and the FNR. The usage 
of soft labels has recently gained traction as they have shown significant advantages~\cite{thiel2008classification,Diaz_2019_CVPR,lukov2022teaching,Hu2017,10095504,Li_2023_CVPR,Yang2023,Cui2023}.  For example, datasets with crowd-sourced annotations can provide soft labels~\cite{ishida2022performance,cifar10h,cifar10s}.
Moreover, soft labels are of utmost importance in label smoothing and knowledge distillation~\cite{szegedy2016rethinking,yuan2023learning,zhang2021delving,zhang2021rethinking,zhang2019your,zhou2021rethinking}. 
Leveraging soft labels to estimate optimal evaluation metrics has recently been studied in~\cite{ishida2022performance,jeong2023demystifying}. In particular, the authors in~\cite{ishida2022performance} studied the Bayes error rate (BER) of binary classification and provided BER estimators using soft labels and noisy labels. 
In~\cite{jeong2023demystifying}, we investigated the estimation of the BER of multi-class classification problems, and we proposed several effective BER estimators. In particular, we proposed denoising methods that minimize the noise in noisy labels, improving the applicability of the BER estimators.

In this paper, we focus on effectively estimating the FPR and FNR in binary classification problems. In particular, since the FPR and FNR are symmetric quantities in terms of the class, we consider the FPR, which readily shows the analysis of the FNR.
We first propose an FPR estimator assuming the knowledge of the class prior probability, and we show that it benefits from several appealing properties, such as unbiasedness, consistency, and asymptotic normality, when the dataset consists of soft labels. 
After that, we remove the assumption on the knowledge of the class prior probability in the estimator and prove that the estimator is still consistent with soft-labeled datasets.
We finally broaden the label type and analyze scenarios in which labels are noisy.
In particular, we focus on the case of additive noise;
we show that binary labels can be viewed as noisy labels. 
The focus is on retrieving soft labels from noisy labels, enabling the use of the proposed FPR estimator. We leverage two methods, a denoising method that we proposed in~\cite{jeong2023demystifying} and the Nadaraya-Watson estimator~\cite{nadaraya1964estimating,watson1964smooth}, to propose an effective FPR estimator for noisy-labeled datasets. In particular, we provide an estimator of the FPR given noisy labels, which is consistent under mild assumptions.

\noindent {\bf{Notation}.} For any $k \in \mathbb{N}$, we define ${[k] := \{1,2, \ldots,k\}}$. For a set $\Xc$, $|\mathcal{X}|$ denotes its cardinality. $\mathbbm{1}\{\Sc\}$ is the indicator function that yields 1 if $\Sc$ is true and 0 otherwise. 
$\varnothing$ is the empty set.
We denote by $\overset{d}{\to}$ and $\overset{p}{\to}$ convergence in distribution and in probability, respectively. 
We use the notion of weak consistency, which is formally defined below.
\begin{definition}
An estimator $T_n$ of a parameter $\theta$ is consistent if it converges in probability to the true value of the parameter, i.e., $\lim_{n \to \infty} \Pr \left( |T_n - \theta|>\epsilon \right )=0$ for all $\epsilon >0$.
\end{definition}


\section{Problem setting}

\subsection{FPR, FNR, and Bayes classifier}
We consider a binary classification task in which 
a feature $\xs\in\Xc$ is classified into a class $\cs\in\{0,1\}$. 
Our goal is to estimate the FPR and the FNR, which are defined as follows,
\begin{align}
	& \rho_{\rm{FP}}(\phi) = \Pr(\phi(\xs) = 1 | \cs = 0),  \label{eq:FPR}\text{ and} \\
	& \rho_{\rm{FN}}(\phi) = \Pr(\phi(\xs)=0 | \cs = 1), \label{eq:FNR}
\end{align}
where $\phi:\Xc\to\{0,1\}$ is a classifier.\footnote{In hypothesis testing, FPR and FNR are referred to as Type-I and Type-II errors, respectively.}
In particular, we seek to estimate FPR and FNR from a dataset $\Dc$ when the Bayes classifier $\phi_{B}$ is employed,\footnote{The Bayes classifier is the optimal classifier with respect to the accuracy. For given $\xs$, the maximum might not be unique. In this case, any sensible tie break can be used (e.g., randomly choosing one of the outcomes).}
with
\begin{equation}\label{eq:bayes_classifier}
	\phi_{B}(\xs)
	= \arg \max_{i\in\{0,1\}} p_{\cs|\xs}( i | \xs),
\end{equation}
where $p_{\cs|\xs}$ is the conditional probability mass function of $\cs$ given $\xs$.
In order to estimate $\rho_{\rm{FP}}(\phi_B)$ and $\rho_{\rm{FN}}(\phi_B)$, we make a typical assumption~\cite{Song2022} that the dataset $\Dc = \{ (\xs_i,\ys_i) \}_{i=1}^n$ is independent and identically distributed (i.i.d.) according to an unknown data distribution $p_{\xs,\ys}$, i.e., $(\xs_i,\ys_i) \overset{i.i.d.}{\sim} p_{\xs,\ys}$, where $\xs_i\in\Xc$ and $\ys_i$ is the label.


\subsection{Soft labels}
We consider soft labels, also known as probabilistic labels~\cite{huai2020learning,peng2014learning}. In general, a soft label $\ys$ is a real value between $0$ and $1$, i.e., $\ys \in [0,1]$.
Here, we assume that a label is soft if it represents the posterior probability~\cite{dao2021knowledge,grossmann2022beyond,pmlr-v139-menon21a,peterson2019human,zhou2021rethinking,NEURIPS2018_bd135462,wang2023binary}, as formally defined next.
\begin{definition}
\label{def:SoftLabel}
$(\xs_i,\ys_i)\in\Dc$ is a soft labeled data sample if
\begin{align}
\label{eq:SoftLabel}
	\ys_i = p_{\cs|\xs}(1|\xs_i).
\end{align}
\end{definition}

We assume that the label $\ys$ contains the information about the class $\cs$ of the input feature $\xs$. Formally,
we assume that
\begin{subequations}\label{eq:PcDefinition}
\begin{align}
    & p_\cs(0)
    = \Pr(\ys<0.5) >0, \text{ and } \\
    & p_\cs(1)
    = \Pr(\ys \geq 0.5)>0,
\end{align}
\end{subequations}
where $p_\cs(i) = \Pr (\cs=i), i \in \{0,1\}$.\footnote{The assumption that $p_\cs(0)$ and $p_\cs(1)$ are non-zero is indeed necessary to properly define $\rho_{\rm FP}(\phi)$ and $\rho_{\rm FN}(\phi)$.}

\subsection{Noisy labels}
In order to relax the soft label assumption, we introduce noisy labels. We say that a label is a noisy label if the label consists of a soft label and some random label noise, as formally defined next.
\begin{definition}\label{def:noisy_label}
    $(\xs_i,\tilde{\ys}_i)\in\Dc$ is a noisy labeled data sample if
\begin{align}
	\tilde{\ys}_i \sim p_{\tilde{\ys}|\ys=\ys_i},
\end{align}
where $\ys_i$ is the soft label corresponding to $\xs_i$. 
Moreover, we assume a zero-mean noise distribution, i.e., $\EE[\tilde{\ys}|\ys] = \ys$.
\end{definition}

\begin{remark}\label{remark:one_hot}
The modeling through noisy labels can serve multiple purposes. First, the labels can indeed be noisy.  Second, many practical datasets are binary-labeled.  
Viewing binary labels as noisy labels, one can extend the range of techniques that are used for soft labels (see also~\cite{ishida2022performance,jeong2023demystifying}). An example of modeling binary labels as noisy labels is given next,
\begin{align}
    p_{\tilde{\ys}|\ys}
    & = \begin{cases}
        \ys & \text{ if } \tilde{\ys}=1, \\
        1-\ys & \text{ if } \tilde{\ys}=0,
    \end{cases}
\end{align}
for which it is not difficult to see that $\EE[\tilde{\ys}|\ys] = \ys$.
\end{remark}

Throughout the paper, we denote by $\ys$ a soft label and by $\tilde{\ys}$ a noisy label.
We will focus on deriving effective estimators for the FPR
in~\eqref{eq:FPR}.
Since the FNR in~\eqref{eq:FNR} is symmetric to the FPR in terms of the class, our analysis on the FPR can be readily extended to obtain effective estimators of the FNR.

\section{FPR estimate: Soft labels}
In this section, we propose effective estimators of the FPR in~\eqref{eq:FPR}, by assuming that the dataset $\Dc$ contains soft labels as per Definition~\ref{def:SoftLabel}.
We start by observing that we can write the FPR in~\eqref{eq:FPR} evaluated in $\phi_{B}$ in~\eqref{eq:bayes_classifier} as follows,
\begin{align}
\rho_{\rm{FP}}(\phi_{B})  \nonumber
& = \Pr(\phi_{B}(\xs) = 1 | \cs = 0)\nonumber
\\& = \int_{\Xc}  \mathbbm{1} \!\left \{ p_{\cs|\xs}( 1 | \xs) \geq 0.5\right \}  p_{\xs|\cs}({\rm{d}}\xs|0) \nonumber
\\& = \frac{1}{p_\cs(0)} \int_{\Xc}  \mathbbm{1} \! \left \{ p_{\cs|\xs}( 1 | \xs) \geq 0.5\right \} p_{\cs|\xs}(0|\xs) p_{\xs}({\rm{d}}\xs)  \nonumber
%
%
\\& = \frac{1}{p_\cs(0)} \EE \!\left [\mathbbm{1} \!\left \{ \ys \geq 0.5 \right \} (1-\ys) \right ],
\label{eq:MotivationEstimator}
\end{align}
where in the last equality we have used Definition~\ref{def:SoftLabel}.
Based on~\eqref{eq:MotivationEstimator}, we propose a natural estimator of $\rho_{\rm{FP}}(\phi_{B})$, which is formally defined below.
We start by assuming that $p_\cs(i), i \in \{0,1\}$ is known, and we will remove this assumption later. 
%
\begin{definition}
Assume that $p_\cs(0)$ is known. An FPR estimator is defined as follows,
\begin{align}
\label{eq:IdealEstimator}
	\psi_{\rm FP,1}(\Dc) =  \frac{1}{n p_\cs(0)} \sum_{(\xs,\ys)\in\Dc} \mathbbm{1} \{  \ys \geq 0.5 \} (1-\ys).
\end{align}
\end{definition}
The next theorem provides important properties of $\psi_{\rm FP,1}(\Dc)$ in~\eqref{eq:IdealEstimator} under soft labels.

\begin{theorem}
\label{thm:NoiselessIdeal}
Assume that $\Dc$ contains soft labels. Then, $\psi_{\rm FP,1}(\Dc)$ in~\eqref{eq:IdealEstimator} satisfies the following properties:
\begin{enumerate}
	\item {\rm{(Unbiasedness):}} $\EE[\psi_{\rm FP,1}(\Dc)] = \rho_{\rm{FP}}(\phi_B)$;
	\item {\rm{(Convergence rate):}} For any $\delta\in(0,1)$ it holds that
 \begin{equation}
 \label{eq:CovergenceRate}
 | \psi_{\rm FP,1}(\Dc) - \rho_{\rm{FP}}(\phi_B) | < \sqrt{ \frac{\ln(2/\delta)}{8 n p_\cs^2(0)} },
 \end{equation}
 with probability at least $1-\delta$;
	\item {\rm {(Variance):}} It holds that
 \begin{align}
    {\rm{Var}}(\psi_{\rm FP,1}(\Dc))
    & \leq \frac{1}{16n p_\cs^2(0)};
\end{align}
	\item {\rm{(Asymptotic normality):}} As $n\to\infty$, it holds that $\sqrt{n} ( \psi_{\rm FP,1}(\Dc) - \rho_{\rm FP}(\phi_B)) \overset{d}{\to} \Nc(0,{\rm{Var}}(\psi_{\rm FP,1}(\Dc)))$.
\end{enumerate}
\end{theorem}

\begin{proof}
See Section~\ref{sec:ProofThmNoiselessIdeal}.
\end{proof}
The convergence rate result in Theorem~\ref{thm:NoiselessIdeal} readily implies that $ | \psi_{\rm FP,1}(\Dc) - \rho_{\rm{FP}}(\phi_B) | \overset{p}{\to} 0$ as $n\to\infty$, which leads to the following result.
\begin{corollary}
\label{cor:Consistency}
Assume that $\Dc$ contains soft labels. Then, $\psi_{\rm FP,1}(\Dc)$ in~\eqref{eq:IdealEstimator} is a consistent estimator of the FPR.
\end{corollary}
Theorem~\ref{thm:NoiselessIdeal} and Corollary~\ref{cor:Consistency} show several appealing properties of the proposed estimator in~\eqref{eq:IdealEstimator}. However, this estimator assumes the knowledge of $p_\cs(i), i \in \{0,1\}$, which may not be available. 
Motivated by this observation, we next use an estimate of $p_\cs(0)$, which we refer to as $\hat{p}_\cs(0)$. In particular, from~\eqref{eq:PcDefinition} it follows that $p_\cs(0) = \EE \left [ \mathbbm{1}\{\ys<0.5 \}\right ]$ and hence, we propose the following  natural estimate of $p_\cs(0)$,
\begin{align}
\label{eq:EstimatorPc0}
	\hat{p}_\cs(0)
	& = \frac{1}{n} \sum_{(\xs,\ys)\in\Dc} \mathbbm{1}\{ \ys < 0.5 \}.
\end{align}
The above leads to the FPR estimator formally defined below.
\begin{definition}
Assume 
that $\Dc$ contains soft labels. An FPR estimator is defined as follows,
\begin{equation}
\label{eq:SecondFPREst}
    \psi_{\rm FP,2}(\Dc) 
    =  \frac{ \sum_{(\xs,\ys)\in\Dc} \mathbbm{1} \{  \ys \geq 0.5 \} (1-\ys)}{ \max \left \{\epsilon, \sum_{(\xs,\ys)\in\Dc} \mathbbm{1}\{ \ys < 0.5 \} \right \}},
\end{equation}
where $\epsilon >0$ is an arbitrarily small parameter.

\end{definition}
We note that $\psi_{\rm FP,2}(\Dc)$ in~\eqref{eq:SecondFPREst} is clearly a biased estimator of the FPR.
Nevertheless, the next theorem shows that $\psi_{\rm FP,2}(\Dc)$ in~\eqref{eq:SecondFPREst} is a consistent estimator of the FPR. 
\begin{theorem}
\label{thm:Noiseless2}
Let $\Dc$ contain soft labels. Then, $\psi_{\rm FP,2}(\Dc)$ in~\eqref{eq:SecondFPREst} is a consistent estimator of the FPR.
\end{theorem}
\begin{proof}
See Section~\ref{sec:ProofThmNoiseless2}. 

\end{proof}

\section{FPR estimate: Noisy labels}

In this section, we consider the case of noisy labels defined in Definition~\ref{def:noisy_label}.
In particular, we focus on the practically relevant case of additive noise, i.e., $\tilde{\ys} = \ys + \zs$, where $\zs$ is some random noise.
Motivated by the fact that Theorem~\ref{thm:NoiselessIdeal} and Theorem~\ref{thm:Noiseless2} demonstrate the effectiveness of using soft labels in estimating the FPR, we here propose denoising methods.
To properly denoise $\tilde{\ys}$, we first define 
\begin{align}\label{eq:Xc_y}
    \Xc_{\ys_i} 
    & = \{\xs\in\Xc : p_{\cs|\xs}(1|\xs) = \ys_i\}.
\end{align}
From the zero-mean noise assumption (i.e, $\EE[\tilde{\ys}|\ys] = \ys$) in Definition~\ref{def:noisy_label}, we write a soft label as
\begin{align}\label{eq:denoise_expectation}
    \ys_i
    & = \EE[\tilde{\ys} | \ys = \ys_i] \nonumber \\
    & = \EE[\tilde{\ys} | \ys = \ys_i, \xs\in\Xc_{\ys_i}] \nonumber \\
    & = \EE[\tilde{\ys} | \xs \in \Xc_{\ys_i}] ,
\end{align}
where the second equality is due to the Markov chain $\xs \to \ys \to \tilde{\ys}$.
Using the relationship between $\ys_i$ and $\xs$ in~\eqref{eq:denoise_expectation}, we next propose label denoising methods that mitigate the label noise by leveraging the feature data samples.

The crucial part in~\eqref{eq:denoise_expectation} is the condition $\xs\in\Xc_{\ys_i}$. However, it is in general difficult to know the set $\Xc_{\ys_i}$ explicitly from the dataset $\Dc$. 
Instead, we replace the condition $\xs\in\Xc_{\ys_i}$ in~\eqref{eq:denoise_expectation} with the condition $\xs = \xs_i$ since $\xs_i\in\Xc_{\ys_i}$, which yields
\begin{align}\label{eq:denoising_expectation2}
    \ys_i
    & = \EE[\tilde{\ys} | \xs  = \xs_i] .
\end{align}
For the case of finite $\Xc$, in~\cite{jeong2023demystifying} we recently proposed an unbiased and consistent denoising method, and we showed its convergence rate.
\begin{lemma}[Theorem~3 in~\cite{jeong2023demystifying}]\label{lem:denoise_1}
Let $\Xc$ be a finite set, and assume that the noisy label $\tilde{\ys} = \ys + \zs$, where $\zs$ is random noise, is bounded as $\tilde{\ys}\in[a,b]$ with finite $a$ and $b$.
Let $\Dc = \tilde{\Dc} \cup (\xs_i,\tilde{\ys}_i)$ with $(\xs_i,\tilde{\ys}_i) \notin \tilde{\Dc}$ and
consider the following denoised label for $(\xs_i,\ys_i)$:
\begin{align}\label{eq:denoise}
    \mathsf{dn}(\xs_i,\tilde{\ys}_i ; \tilde{\Dc})
    & = \frac{ \sum_{(\xs,\tilde{\ys})\in \Dc} \mathbbm{1}\{\xs_i = \xs\}\tilde{\ys}}{\sum_{(\xs,\tilde{\ys})\in \Dc} \mathbbm{1}\{\xs_i = \xs\}}.
\end{align}
Then, the following properties hold:
\begin{enumerate}
    \item {\rm (Unbiasedness):} $\EE[\mathsf{dn}(\xs_i,\tilde{\ys}_i ; \tilde{\Dc})] = \ys_i$;
    \item {\rm (Consistency):} $\mathsf{dn}(\xs_i,\tilde{\ys}_i ; \tilde{\Dc}) \overset{p}{\to} \ys_i$ as $n\to\infty$;
    \item {\rm (Convergence rate):} For any $\delta\in(0,1)$, with probability at least $1-\delta$, it holds that
    \begin{equation}
    |\mathsf{dn}(\xs_i,\tilde{\ys}_i ; \tilde{\Dc}) - \ys| < \sqrt{\frac{(\frac{1}{2}+b-a)^2}{2 n_{\xs_i}} \ln\frac{2}{\delta} },
    \end{equation}
    where $n_{\xs_i} = \sum_{(\xs,\tilde{\ys})\in\Dc} \mathbbm{1}\{\xs_i = \xs\}$.
\end{enumerate}

\end{lemma}

Leveraging~\eqref{eq:denoise}, we now propose a consistent estimator of $\rho_{\rm{FP}}(\phi_B)$ for noisy label datasets. In particular, we incorporate the FPR estimator in~\eqref{eq:SecondFPREst} with the denoising method in~\eqref{eq:denoise}, and we show its consistency in estimating the FPR.
\begin{theorem}\label{thm:denoising_consistency}
Let $\Xc$ be a finite set, and let $\Dc$ be a dataset with continuous noisy labels $\tilde{\ys} = \ys + \zs$, where $\zs$ is some random noise, bounded as $\tilde{\ys}\in[a,b]$ with finite $a$ and $b$. 
Let $\Dc_1$ and $\Dc_2$ be an arbitrary partition of $\Dc$ with ratio $r=\frac{|\Dc_1|}{|\Dc|}$.
Consider the following estimator:
\begin{align}\label{eq:denoise_est}
	& \Psi_{\rm FP,1}(\Dc) \nonumber \\
    & =   \frac{ \sum\limits_{(\xs,\tilde{\ys})\in\Dc_1} \! \mathbbm{1} \{  \mathsf{dn}(\xs,\tilde{\ys}; \Dc_2) \geq 0.5 \} (1-\mathsf{dn}(\xs,\tilde{\ys}; \Dc_2))}{ \max\left\{ \epsilon, \sum_{(\xs,\tilde{\ys})\in\Dc_1} \mathbbm{1}\{ \mathsf{dn}(\xs ,\tilde{\ys}; \Dc_2) < 0.5 \} \right\}}, 
\end{align}
where $\epsilon >0$ is an arbitrarily small parameter.
Then, for any $r\in(0,1)$, $\Psi_{\rm FP,1}(\Dc)$ is a consistent estimator of the FPR.
\end{theorem}
    
\begin{proof}
See Section~\ref{sec:proof_dn_estimate}.
\end{proof}

Theorem~\ref{thm:denoising_consistency} shows the effectiveness of $\Psi_{\rm FP,1}(\Dc)$ in~\eqref{eq:denoise_est} in estimating the FPR given a noisy labeled dataset, when the feature space is finite and the noisy labels are continuous. 

However, $\Psi_{\rm FP,1}(\Dc)$ in~\eqref{eq:denoise_est} also suffers from some limitations. For instance, if the number of data samples is not sufficiently large, $\Psi_{\rm FP,1}(\Dc)$ in~\eqref{eq:denoise_est} may result in a poor estimate of the FPR.
Moreover, the finite sample space assumption does not hold for some practical problems, especially when features have continuous values.
To counter these limitations, we pose the denoising problem as a non-parametric estimation of the conditional expectation in~\eqref{eq:denoising_expectation2}. In particular, we reconstruct the soft label by taking the local average around $\xs=\xs_i$, using the Nadaraya-Watson (NW) estimator~\cite{nadaraya1964estimating,watson1964smooth}, which is formally defined below. 
\begin{definition}\label{def:NW}
Let $\Dc = \tilde{\Dc} \cup (\xs_i,\tilde{\ys}_i)$ with $(\xs_i,\tilde{\ys}_i) \notin \tilde{\Dc}$.
The NW estimator~\cite{nadaraya1964estimating,watson1964smooth} of $\EE[\tilde{\ys} | \xs ]$ at $(\xs_i,\tilde{\ys}_i)$ is given~by
\begin{align}\label{eq:nw}
    \mathsf{nw}(\xs_i,\tilde{\ys}_i;\tilde{\Dc})
     = \frac{\sum_{(\xs,\tilde{\ys})\in\Dc} K\!\left(\frac{d(\xs_i,\xs)}{h}\right)\tilde{\ys}}{\sum_{(\xs,\tilde{\ys})\in\Dc} K\!\left(\frac{d(\xs_i,\xs)}{h}\right)} ,
\end{align}
where $d(\cdot,\cdot):\Xc\times \Xc \to \RR_+$ is a metric, $h$ is the bandwidth, and $K(\cdot)$ is a kernel supported on $[0,1]$ satisfying the following conditions: 1) it is strictly decreasing; 2) it is Lipschitz continuous; and 3) $\exists \theta, \forall t\in[0,1], 0<\theta<-K^\prime(t)$.
\end{definition}
With the NW estimator in~\eqref{eq:nw} of the soft labels, we can now effectively estimate the FPR by leveraging the estimator in~\eqref{eq:SecondFPREst} with the NW estimator. The next theorem shows its consistency under some assumptions.
\begin{theorem}\label{thm:NW_consistency}
Assume that:
\begin{enumerate}
    \item $\Xc$ is a compact subset of the metric space $\RR^m$ with the metric $d$ in Definition~\ref{def:NW};
    \item H\"older condition: there exist $C<\infty$ and $\beta>0$ such that for all $(\xs_i,\xs_j)\in\Xc^2$, $|p_{\cs|\xs}(1|\xs_i) - p_{\cs|\xs}(1|\xs_j)| \leq C d(\xs_i,\xs_j)^\beta$;
    \item $\Dc$ is a dataset with continuous noisy labels bounded as $\tilde{\ys}\in[a,b]$ with finite $a$ and $b$.
\end{enumerate}
Consider the following estimator:
\begin{align}\label{eq:NW_est}
	& \Psi_{\rm FP,2}(\Dc)  \nonumber \\
    & =  \frac{\sum\limits_{(\xs,\tilde{\ys})\in\Dc_1} \!\!\! \mathbbm{1} \{  \mathsf{nw}(\xs,\tilde{\ys};\Dc_2) \geq 0.5 \} (1 - \mathsf{nw}(\xs,\tilde{\ys};\Dc_2))}{ \max\left\{ \epsilon, \sum_{(\xs,\tilde{\ys})\in\Dc_1} \mathbbm{1}\{ \mathsf{nw}(\xs,\tilde{\ys};\Dc_2) < 0.5 \} \right\}} , 
\end{align}
where $\epsilon >0$ is an arbitrarily small parameter, and $\Dc_1$ and $\Dc_2$ form an arbitrary partition of $\Dc$ with ratio 
$r=\frac{|\Dc_1|}{|\Dc|}$.
Then, when 
$n\to \infty$ and $h\to0$ with $\frac{\ln n}{nh^m}\to0$,
it holds that $\Psi_{\rm FP,2}(\Dc)$ in~\eqref{eq:NW_est} is a consistent estimator of $\rho_{\rm FP}(\phi_B)$.
\end{theorem}
\begin{proof}
See Section~\ref{sec:NW_cons_proof}.
\end{proof}


We conclude this section with the next example, which verifies that Theorem~\ref{thm:NW_consistency} is a generalized version of Theorem~\ref{thm:denoising_consistency}.
\begin{example}
Let us choose the metric $d$ as follows,
\begin{align}\label{eq:metric_recover_dn}
    d(\us,\vs)
    & = \begin{cases}
        0 & \text{ if } \us = \vs, \\
        1 & \text{ if } \us \neq \vs.
    \end{cases}
\end{align}
If we choose $K(t) = 1-t$ for all $t\in[0,1]$ and $h=1$, the NW estimator in~\eqref{eq:nw} with the metric $d$ in~\eqref{eq:metric_recover_dn} retrieves the denoising method $\mathsf{dn}(\cdot,\cdot;\cdot)$ in~\eqref{eq:denoise}.
Since a finite set is compact and the H\"older condition holds with the metric $d$ in~\eqref{eq:metric_recover_dn}, Theorem~\ref{thm:NW_consistency} shows the consistency of $\Psi_{\rm FP,1}(\Dc)$ in~\eqref{eq:denoise_est}.
\end{example}

\section{Proofs of Main Results}

\subsection{Proof of Theorem~\ref{thm:NoiselessIdeal}}
\label{sec:ProofThmNoiselessIdeal}

1) {\em (Unbiasedness).} We have that
\begin{align}\label{eq:base_est_proof1}
	\EE[\psi_{\rm FP,1}(\Dc)] \nonumber
	& \stackrel{{\rm{(a)}}}{=} \EE\left[\frac{1}{n p_\cs(0)} \sum_{(\xs,\ys)\in\Dc} \mathbbm{1} \{  \ys \geq 0.5 \} (1-\ys) \right] \nonumber\\
	& \stackrel{{\rm{(b)}}}{=} \frac{1}{p_\cs(0)}  \EE[\mathbbm{1} \{  \ys \geq 0.5 \} (1-\ys)] \nonumber\\
 & \stackrel{{\rm{(c)}}}{=} \rho_{\rm{FP}}(\phi_{B}),
\end{align}
where the labeled equalities follow from:
$\rm{(a)}$ using the expression of $\psi_{\rm FP,1}(\Dc)$ in~\eqref{eq:IdealEstimator};
$\rm{(b)}$ the i.i.d. assumption on the data samples;
and $\rm{(c)}$ using~\eqref{eq:MotivationEstimator}.

2) {\it (Convergence rate).}
Since $\mathbbm{1}\{\ys\geq 0.5\} (1-\ys) \in [0,0.5]$, the Hoeffding's inequality yields
\begin{align}
	\Pr( | \psi_{\rm FP,1}(\Dc) - \rho_{\rm{FP}}(\phi_B) | \geq t )
	& \leq 2 {\rm{e}}^{-\frac{2 n^2   t^2 p_\cs^2(0)}{\sum_{i=1}^n (0.5)^2 }} \nonumber  \\
	& = 2 {\rm{e}}^{- 8 n   t^2 p_\cs^2(0)} .
\end{align}
For any $\delta\in (0,1)$, let $t = \sqrt{ \frac{1}{8 n p_\cs^2(0)} \ln(\frac{2}{\delta})} $. Then, we obtain
\begin{align}
	\Pr\left( | \psi_{\rm FP,1}(\Dc) - \rho_{\rm{FP}}(\phi_B) | \geq \sqrt{ \frac{\ln(\frac{2}{\delta})}{8 n p_\cs^2(0)} } \right)
	& \leq \delta,
\end{align}
which implies that for any $\delta\in(0,1)$, the following holds
\begin{align}
	 | \psi_{\rm FP,1}(\Dc) - \rho_{\rm{FP}}(\phi_B) | 
	 < \sqrt{ \frac{\ln(\frac{2}{\delta})}{8 n p_\cs^2(0)} },
\end{align}
with probability at least $1-\delta$.

3) {\em (Variance).}
We first observe that
\begin{align}
	\var(\psi_{\rm FP,1}(\Dc)) \nonumber
	& = \var\left( \frac{1}{n p_\cs(0)} \sum_{(\xs,\ys)\in\Dc} \mathbbm{1} \{  \ys \geq 0.5 \} (1-\ys) \right) \\
	& = \frac{\var( \mathbbm{1} \{  \ys \geq 0.5 \} (1-\ys) )}{n p_\cs^2(0)} ,
\end{align}
where the last equality follows from the i.i.d. assumption on the data samples. Since $\mathbbm{1} \{  \ys \geq 1-\ys \} (1-\ys) \in [0,0.5]$, applying Popoviciu's inequality on variances~\cite{bhatia2000better} yields
\begin{align}
    \var(\psi_{\rm FP,1}(\Dc))
    & \leq \frac{1}{16n p_\cs^2(0)}.
\end{align}

4) {\em (Asymptotic normality).} The estimator $\psi_{\rm FP,1}(\Dc)$ in~\eqref{eq:IdealEstimator} is the sample mean of $\frac{\mathbbm{1}\{\ys\geq 1-\ys\} (1-\ys)}{p_{\cs}(0)}$. 
Hence, due to the central limit theorem, we obtain that as $n\to\infty$,
\begin{align}
	\sqrt{n}(\psi_{\rm FP,1}(\Dc) - \rho_{{\rm{FP}}}(\phi_B))
	& \overset{d}{\to} \Nc(0,\var(\psi_{\rm FP,1}(\Dc))).
\end{align}
This concludes the proof of Theorem~\ref{thm:NoiselessIdeal}.

\subsection{Proof of Theorem~\ref{thm:Noiseless2}}
\label{sec:ProofThmNoiseless2}
Recall that $\Dc = \{ (\xs_i,\ys_i) \}_{i=1}^n$ being distributed according to an unknown data distribution $p_{\xs,\ys}$, i.e., $(\xs_i,\ys_i) \overset{i.i.d.}{\sim} p_{\xs,\ys}$.

Our goal is to show that as $n\to\infty$, it holds that
\begin{equation}
\label{eq:GoalCons2}
\psi_{\rm FP,2}(\Dc) \overset{p}{\to} \rho_{\rm FP}(\phi_B).
\end{equation}
We start by analyzing the numerator of $\psi_{\rm FP,2}(\Dc)$ in~\eqref{eq:SecondFPREst}. Multiplying the numerator by $1/n$, we have that
\begin{align} 
&\lim_{n \to \infty} \frac{1}{n}\sum_{i=1}^n \mathbbm{1} \{  \ys_i \geq 0.5 \} (1-\ys_i) \nonumber
\\& = \EE[ \mathbbm{1} \{  p_{\cs|\xs}(1|\xs) \geq  0.5 \}  p_{\cs|\xs}(0|\xs) ]\nonumber
\\& = p_{\cs}(0) \rho_{\rm FP}(\phi_B),
\label{eq:NumPhi2}
\end{align}
where the first equality follows from the law of large numbers and the last equality uses~\eqref{eq:MotivationEstimator}.
We now analyze the denominator of $\psi_{\rm FP,2}(\Dc)$ in~\eqref{eq:SecondFPREst}.
Before taking the limit, since $\epsilon>0$ is arbitrarily small, we set $\epsilon={\rm{e}}^{-n}$. Multiplying the denominator by $1/n$, we have that
\begin{align}
\lim_{n \to \infty} \max \left \{ \frac{{\rm{e}}^{-n}}{n} , \sum_{i =1}^n \frac{\mathbbm{1}\{ \ys_i \!<\! 0.5 \}}{n} \right \} \nonumber
& = \max \left \{  0,\Pr(\ys \!<\! 0.5 ) \right \}\nonumber
\\ &  = p_\cs(0),
\label{eq:DenPhi2}
\end{align}
where the first equality follows from the law of large numbers and the last equality follows from~\eqref{eq:PcDefinition}.  
Since the denominator of $\psi_{\rm FP,2}(\Dc)$ multiplied by $1/n$ is at least $\frac{\epsilon}{n}>0$, we can use the continuous mapping theorem~\cite{continuous_mapping},
which gives that
\begin{align}
    \psi_{\rm FP,2}(\Dc)
    & \overset{p}{\to} \rho_{\rm{FP}}(\phi_B)
\end{align}
as $n\to\infty$,
where we have used~\eqref{eq:NumPhi2} and~\eqref{eq:DenPhi2}.
This shows~\eqref{eq:GoalCons2} and concludes the proof of Theorem~\ref{thm:Noiseless2}.

\subsection{Proof of Theorem~\ref{thm:denoising_consistency}}\label{sec:proof_dn_estimate}
Our goal is to show that as $n\to\infty$, it holds that
\begin{equation}
\label{eq:GoalConsNoisy}
\Psi_{\rm FP,1}(\Dc) \overset{p}{\to} \rho_{\rm FP}(\phi_B).
\end{equation}
With the ratio $r=\frac{|\Dc_1|}{|\Dc|}$, we have that $|\Dc_1| = rn$ and $|\Dc_2|=(1-r)n$.
We index the data samples in $\Dc_1$ with $i$ and those in $\Dc_2$ with $j$, i.e., $\Dc_1 = \{(\xs_i,\tilde{\ys}_i): i\in[rn]\}$ and $\Dc_2 = \{(\xs_j,\tilde{\ys}_j):j\in[(1-r)n)]\}$.
We start by analyzing the numerator of $\Psi_{\rm FP,1}(\Dc)$ in~\eqref{eq:denoise_est}. Multiplying the numerator by $1/rn$, we have that
\begin{align}
    & \lim_{n\to\infty}\frac{1}{rn}\sum_{i=1}^{rn} \mathbbm{1} \{  \mathsf{dn}(\xs_i,\tilde{\ys}_i; \Dc_2) \geq 0.5 \} (1-\mathsf{dn}(\xs_i,\tilde{\ys}_i; \Dc_2)) \nonumber\\
    & \overset{\rm (a)}{=} \lim_{n\to\infty}\frac{1}{nr}\sum_{i=1}^{nr} \mathbbm{1} \{  \ys_i \geq 0.5 \} (1-\ys_i) \nonumber\\
    & \overset{\rm (b)}{=} \EE[ \mathbbm{1} \{  \ys \geq 0.5 \} (1-\ys) ],
    \label{eq:NumDenoiseEst}
\end{align}
where the labeled equalities follow from: 
$\rm (a)$ the fact that $\Dc_1\cap\Dc_2 = \varnothing$ and leveraging Lemma~\ref{lem:denoise_1} with the continuous mapping theorem~\cite{continuous_mapping} that is verifiable with the facts that $\mathbbm{1}\{t\geq0.5\}(1-t)$ has only one discontinuity at $t=0.5$ and $\Pr(\mathsf{dn}(\xs_i,\tilde{\ys}_i;\Dc_2) = 0.5) = 0$ because $\tilde{\ys}$ is a continuous random variable; 
and $\rm (b)$ the fact that the soft labels $\ys_i$ are i.i.d. and using the law of large numbers.

We now analyze the denominator of $\Psi_{\rm FP,1}(\Dc)$ in~\eqref{eq:denoise_est}. Since $\epsilon$ is an arbitrarily small parameter, we set $\epsilon = {\rm{e}}^{-rn}$.
By multiplying the denominator by $1/rn$, we have that
\begin{align}
    & \lim_{n\to\infty} \max\left\{ \frac{{\rm{e}}^{-rn}}{rn}, \frac{1}{rn}\sum_{i=1}^{rn} \mathbbm{1}\{ \mathsf{dn}(\xs_i,\tilde{\ys}_i; \Dc_2) < 0.5 \} \right\} \nonumber \\
    & = \max\left\{ 0, \EE[ \mathbbm{1}\{ \ys < 0.5 \} ] \right\} \nonumber \\
    & =  p_\cs(0) ,
    \label{eq:DenDenoiseEst}
\end{align}
where the first equality follows by using similar steps as in~\eqref{eq:NumDenoiseEst} and the last equality follows from~\eqref{eq:PcDefinition}.

Since the denominator of $\Psi_{\rm FP,1}(\Dc)$ is non-zero, we can use the continuous mapping theorem~\cite{continuous_mapping}, which gives that 
\begin{align}
    \Psi_{\rm FP,1}(\Dc)
    & \overset{p}{\to} \frac{\EE[ \mathbbm{1} \{  \ys \geq 0.5 \} (1-\ys) ]}{p_\cs(0)} \nonumber \\
    & = \rho_{\rm{FP}}(\phi_B),
\end{align}
as $n\to\infty$, where we have put together~\eqref{eq:NumDenoiseEst} and~\eqref{eq:DenDenoiseEst}, and the last equality follows from~\eqref{eq:MotivationEstimator}.
This shows~\eqref{eq:GoalConsNoisy} and concludes the proof of Theorem~\ref{thm:denoising_consistency}.

\subsection{Proof of Theorem~\ref{thm:NW_consistency}}
\label{sec:NW_cons_proof}
The proof leverages the following lemma from~\cite{ferraty2004nonparametric}.
\begin{lemma}[Corollary~4.3 in~\cite{ferraty2004nonparametric}]\label{lem:nw_convergence}
Assume that:
\begin{enumerate}
    \item $\Xc$ is a compact subset of the metric space $\RR^m$ with the metric $d$ in Definition~\ref{def:NW};
    \item H\"older condition: there exist $C<\infty$ and $\beta>0$ such that for all $(\xs_i,\xs_j)\in\Xc^2$, $|p_{\cs|\xs}(1|\xs_i) - p_{\cs|\xs}(1|\xs_j)| \leq C d(\xs_i,\xs_j)^\beta$;
    \item Bounded second moment of the noisy label: $\EE[\tilde{\ys}^2] < \infty$;
    \item $\sup_{i\neq j} \EE[|\tilde{\ys}_i \tilde{\ys}_j| \mid \xs_i,\xs_j] \leq M < \infty$;
    \item There exists $\kappa>0$ such that $\inf_{\xs\in\Xc} p_\xs(\xs)\geq \kappa$, where $p_\xs = \frac{d F_\xs}{d \mu}$ is the Radon-Nikodym derivative of the distribution function $F_\xs$ with respect to the Lebesgue measure $\mu$ on $\RR^m$.
\end{enumerate}
Then, it holds that
\begin{align}
    \sup_{\xs\in\Xc} | \ys - \mathsf{nw}(\xs,\tilde{\ys};\tilde{\Dc})| = O(h^\beta) + O\left(\sqrt{\frac{\ln n}{nh^m}}\right) , a.s.
\end{align}
\end{lemma}
Assumptions 1) and 2) in Lemma~\ref{lem:nw_convergence} are assumed to hold in Theorem~\ref{thm:NW_consistency}. 
Since the noisy label is bounded as $\tilde{\ys}\in[a,b]$, we have $\EE[\tilde{\ys}^2] \leq \max\{a^2,b^2\} < \infty$, which satisfies assumption~3) in Lemma~\ref{lem:nw_convergence}. The boundedness of $\tilde{\ys}$ also implies that $\EE[|\tilde{\ys}_i \tilde{\ys}_j| \mid \xs_i,\xs_j] \leq \max\{a^2, b^2\} < \infty$, which satisfies assumption 4) in Lemma~\ref{lem:nw_convergence}.
Since $\Xc$ is the support of $\xs$, assumption 5) in Lemma~\ref{lem:nw_convergence} holds. 

Leveraging Lemma~\ref{lem:nw_convergence}, we have that as $n\to\infty$ and $h\to0$ with $\frac{\ln n}{nh^m}\to0$, for any $\xs\in\Xc$, it holds that
\begin{align}
    \mathsf{nw}(\xs,\tilde{\ys};\tilde{\Dc}) \overset{p}{\to} \ys.
\end{align}
The proof of Theorem~\ref{thm:NW_consistency} is concluded by using the above fact and following the same steps as in the proof of Theorem~\ref{thm:denoising_consistency} (by only replacing $\mathsf{dn}(\xs,\tilde{\ys};\tilde{\Dc})$ with $\mathsf{nw}(\xs,\tilde{\ys};\tilde{\Dc})$).

\bibliographystyle{IEEEtran}
\bibliography{Bayes_error}

\begin{thebibliography}{10}
\providecommand{\url}[1]{#1}
\csname url@samestyle\endcsname
\providecommand{\newblock}{\relax}
\providecommand{\bibinfo}[2]{#2}
\providecommand{\BIBentrySTDinterwordspacing}{\spaceskip=0pt\relax}
\providecommand{\BIBentryALTinterwordstretchfactor}{4}
\providecommand{\BIBentryALTinterwordspacing}{\spaceskip=\fontdimen2\font plus
\BIBentryALTinterwordstretchfactor\fontdimen3\font minus
  \fontdimen4\font\relax}
\providecommand{\BIBforeignlanguage}[2]{{%
\expandafter\ifx\csname l@#1\endcsname\relax
\typeout{** WARNING: IEEEtran.bst: No hyphenation pattern has been}%
\typeout{** loaded for the language `#1'. Using the pattern for}%
\typeout{** the default language instead.}%
\else
\language=\csname l@#1\endcsname
\fi
#2}}
\providecommand{\BIBdecl}{\relax}
\BIBdecl

\bibitem{mouliou2021false}
D.~S. Mouliou and K.~I. Gourgoulianis, ``False-positive and false-negative
  {COVID}-19 cases: Respiratory prevention and management strategies,
  vaccination, and further perspectives,'' \emph{Expert Review of Respiratory
  Medicine}, vol.~15, no.~8, pp. 993--1002, 2021.

\bibitem{hammer2022factors}
M.~M. Hammer, S.~C. Byrne, and C.~Y. Kong, ``Factors influencing the false
  positive rate in {CT} lung cancer screening,'' \emph{Academic Radiology},
  vol.~29, pp. S18--S22, 2022.

\bibitem{bolton2002statistical}
R.~J. Bolton and D.~J. Hand, ``Statistical fraud detection: A review,''
  \emph{Statistical Science}, vol.~17, no.~3, pp. 235--255, 2002.

\bibitem{abdallah2016fraud}
A.~Abdallah, M.~A. Maarof, and A.~Zainal, ``Fraud detection system: A survey,''
  \emph{Journal of Network and Computer Applications}, vol.~68, pp. 90--113,
  2016.

\bibitem{fathy2020learning}
Y.~Fathy, M.~Jaber, and A.~Brintrup, ``Learning with imbalanced data in smart
  manufacturing: A comparative analysis,'' \emph{IEEE Access}, vol.~9, pp.
  2734--2757, 2020.

\bibitem{garcia2004convolutional}
C.~Garcia and M.~Delakis, ``Convolutional face finder: A neural architecture
  for fast and robust face detection,'' \emph{IEEE Transactions on Pattern
  Analysis and Machine Intelligence}, vol.~26, no.~11, pp. 1408--1423, 2004.

\bibitem{larriva2020evaluation}
X.~A. Larriva-Novo, M.~Vega-Barbas, V.~A. Villagr{\'a}, and M.~S. Rodrigo,
  ``Evaluation of cybersecurity data set characteristics for their
  applicability to neural networks algorithms detecting cybersecurity
  anomalies,'' \emph{IEEE Access}, vol.~8, pp. 9005--9014, 2020.

\bibitem{davis2006relationship}
J.~Davis and M.~Goadrich, ``The relationship between precision-recall and {ROC}
  curves,'' in \emph{23rd International Conference on Machine learning}, 2006,
  pp. 233--240.

\bibitem{hastie2009elements}
T.~Hastie, R.~Tibshirani, J.~H. Friedman, and J.~H. Friedman, \emph{The
  Elements of Statistical Learning: Data Mining, Inference, and
  Prediction}.\hskip 1em plus 0.5em minus 0.4em\relax Springer, 2009, vol.~2.

\bibitem{thiel2008classification}
C.~Thiel, ``Classification on soft labels is robust against label noise,'' in
  \emph{International Conference on Knowledge-Based and Intelligent Information
  and Engineering Systems}.\hskip 1em plus 0.5em minus 0.4em\relax Springer,
  2008, pp. 65--73.

\bibitem{Diaz_2019_CVPR}
R.~Diaz and A.~Marathe, ``Soft labels for ordinal regression,'' in
  \emph{IEEE/CVF Conference on Computer Vision and Pattern Recognition}, June
  2019.

\bibitem{lukov2022teaching}
T.~Lukov, N.~Zhao, G.~H. Lee, and S.-N. Lim, ``Teaching with soft label
  smoothing for mitigating noisy labels in facial expressions,'' in
  \emph{European Conference on Computer Vision}.\hskip 1em plus 0.5em minus
  0.4em\relax Springer, 2022, pp. 648--665.

\bibitem{Hu2017}
N.~Hu, G.~Englebienne, Z.~Lou, and B.~Kröse, ``Learning to recognize human
  activities using soft labels,'' \emph{IEEE Transactions on Pattern Analysis
  and Machine Intelligence}, vol.~39, no.~10, pp. 1973--1984, 2017.

\bibitem{10095504}
I.~Martín-Morató, M.~Harju, P.~Ahokas, and A.~Mesaros, ``Training sound event
  detection with soft labels from crowdsourced annotations,'' in \emph{IEEE
  International Conference on Acoustics, Speech and Signal Processing}, 2023,
  pp. 1--5.

\bibitem{Li_2023_CVPR}
S.~Li, M.~Li, R.~Li, C.~He, and L.~Zhang, ``One-to-few label assignment for
  end-to-end dense detection,'' in \emph{IEEE/CVF Conference on Computer Vision
  and Pattern Recognition}, June 2023, pp. 7350--7359.

\bibitem{Yang2023}
C.~Yang, Z.~An, H.~Zhou, F.~Zhuang, Y.~Xu, and Q.~Zhang, ``Online knowledge
  distillation via mutual contrastive learning for visual recognition,''
  \emph{IEEE Transactions on Pattern Analysis and Machine Intelligence},
  vol.~45, no.~8, pp. 10\,212--10\,227, 2023.

\bibitem{Cui2023}
\BIBentryALTinterwordspacing
J.~Cui, R.~Wang, S.~Si, and C.-J. Hsieh, ``Scaling up dataset distillation to
  {I}mage{N}et-1{K} with constant memory,'' in \emph{40th International
  Conference on Machine Learning}, 2023, pp. 6565--6590. [Online]. Available:
  \url{https://proceedings.mlr.press/v202/cui23e.html}
\BIBentrySTDinterwordspacing

\bibitem{ishida2022performance}
\BIBentryALTinterwordspacing
T.~Ishida, I.~Yamane, N.~Charoenphakdee, G.~Niu, and M.~Sugiyama, ``Is the
  performance of my deep network too good to be true? {A} direct approach to
  estimating the {B}ayes error in binary classification,'' in \emph{11th
  International Conference on Learning Representations}, 2023. [Online].
  Available: \url{https://openreview.net/forum?id=FZdJQgy05rz}
\BIBentrySTDinterwordspacing

\bibitem{cifar10h}
\BIBentryALTinterwordspacing
R.~M. Battleday, J.~C. Peterson, and T.~L. Griffiths, ``Capturing human
  categorization of natural images by combining deep networks and cognitive
  models,'' \emph{Nature Communications}, vol.~11, no.~1, p. 5418, 2020.
  [Online]. Available: \url{https://doi.org/10.1038/s41467-020-18946-z}
\BIBentrySTDinterwordspacing

\bibitem{cifar10s}
K.~M. Collins, U.~Bhatt, and A.~Weller, ``Eliciting and learning with soft
  labels from every annotator,'' in \emph{AAAI Conference on Human Computation
  and Crowdsourcing}, vol.~10, 2022.

\bibitem{szegedy2016rethinking}
C.~Szegedy, V.~Vanhoucke, S.~Ioffe, J.~Shlens, and Z.~Wojna, ``Rethinking the
  inception architecture for computer vision,'' in \emph{IEEE/CVF Conference on
  Computer Vision and Pattern Recognition}, 2016, pp. 2818--2826.

\bibitem{yuan2023learning}
\BIBentryALTinterwordspacing
H.~Yuan, Y.~Shi, N.~Xu, X.~Yang, X.~Geng, and Y.~Rui, ``Learning from biased
  soft labels,'' in \emph{37th Conference on Neural Information Processing
  Systems}, 2023. [Online]. Available:
  \url{https://openreview.net/forum?id=gevmGxsTSI}
\BIBentrySTDinterwordspacing

\bibitem{zhang2021delving}
C.-B. Zhang, P.-T. Jiang, Q.~Hou, Y.~Wei, Q.~Han, Z.~Li, and M.-M. Cheng,
  ``Delving deep into label smoothing,'' \emph{IEEE Transactions on Image
  Processing}, vol.~30, pp. 5984--5996, 2021.

\bibitem{zhang2021rethinking}
H.~Zhang, P.~Koniusz, S.~Jian, H.~Li, and P.~H. Torr, ``Rethinking class
  relations: Absolute-relative supervised and unsupervised few-shot learning,''
  in \emph{IEEE/CVF Conference on Computer Vision and Pattern Recognition},
  2021, pp. 9432--9441.

\bibitem{zhang2019your}
L.~Zhang, J.~Song, A.~Gao, J.~Chen, C.~Bao, and K.~Ma, ``Be your own teacher:
  Improve the performance of convolutional neural networks via self
  distillation,'' in \emph{IEEE/CVF International Conference on Computer
  Vision}, 2019, pp. 3713--3722.

\bibitem{zhou2021rethinking}
\BIBentryALTinterwordspacing
H.~Zhou, L.~Song, J.~Chen, Y.~Zhou, G.~Wang, J.~Yuan, and Q.~Zhang,
  ``Rethinking soft labels for knowledge distillation: A
  bias{\textendash}variance tradeoff perspective,'' in \emph{9th International
  Conference on Learning Representations}, 2021. [Online]. Available:
  \url{https://openreview.net/forum?id=gIHd-5X324}
\BIBentrySTDinterwordspacing

\bibitem{jeong2023demystifying}
\BIBentryALTinterwordspacing
M.~Jeong, M.~Cardone, and A.~Dytso, ``Demystifying the optimal performance of
  multi-class classification,'' in \emph{37th Conference on Neural Information
  Processing Systems}, 2023. [Online]. Available:
  \url{https://openreview.net/forum?id=p9k5MS0JAL}
\BIBentrySTDinterwordspacing

\bibitem{nadaraya1964estimating}
E.~A. Nadaraya, ``On estimating regression,'' \emph{Theory of Probability \&
  Its Applications}, vol.~9, no.~1, pp. 141--142, 1964.

\bibitem{watson1964smooth}
G.~S. Watson, ``Smooth regression analysis,'' \emph{Sankhy{\=a}: The Indian
  Journal of Statistics, Series A}, pp. 359--372, 1964.

\bibitem{Song2022}
H.~Song, M.~Kim, D.~Park, Y.~Shin, and J.-G. Lee, ``Learning from noisy labels
  with deep neural networks: A survey,'' \emph{IEEE Transactions on Neural
  Networks and Learning Systems}, pp. 1--19, 2022.

\bibitem{huai2020learning}
M.~Huai, C.~Miao, Y.~Li, Q.~Suo, L.~Su, and A.~Zhang, ``Learning distance
  metrics from probabilistic information,'' \emph{ACM Transactions on Knowledge
  Discovery from Data}, vol.~14, no.~5, pp. 1--33, 2020.

\bibitem{peng2014learning}
P.~Peng, R.~C.-W. Wong, and P.~S. Yu, ``Learning on probabilistic labels,'' in
  \emph{SIAM International Conference on Data Mining}.\hskip 1em plus 0.5em
  minus 0.4em\relax SIAM, 2014, pp. 307--315.

\bibitem{dao2021knowledge}
\BIBentryALTinterwordspacing
T.~Dao, G.~M. Kamath, V.~Syrgkanis, and L.~Mackey, ``Knowledge distillation as
  semiparametric inference,'' in \emph{9th International Conference on Learning
  Representations}, 2021. [Online]. Available:
  \url{https://openreview.net/forum?id=m4UCf24r0Y}
\BIBentrySTDinterwordspacing

\bibitem{grossmann2022beyond}
V.~Grossmann, L.~Schmarje, and R.~Koch, ``Beyond hard labels: Investigating
  data label distributions,'' \emph{arXiv preprint arXiv:2207.06224}, 2022.

\bibitem{pmlr-v139-menon21a}
\BIBentryALTinterwordspacing
A.~K. Menon, A.~S. Rawat, S.~Reddi, S.~Kim, and S.~Kumar, ``A statistical
  perspective on distillation,'' in \emph{38th International Conference on
  Machine Learning}, 2021, pp. 7632--7642. [Online]. Available:
  \url{https://proceedings.mlr.press/v139/menon21a.html}
\BIBentrySTDinterwordspacing

\bibitem{peterson2019human}
J.~C. Peterson, R.~M. Battleday, T.~L. Griffiths, and O.~Russakovsky, ``Human
  uncertainty makes classification more robust,'' in \emph{IEEE/CVF
  International Conference on Computer Vision}, 2019, pp. 9617--9626.

\bibitem{NEURIPS2018_bd135462}
\BIBentryALTinterwordspacing
T.~Ishida, G.~Niu, and M.~Sugiyama, ``Binary classification from
  positive-confidence data,'' in \emph{30th Conference on Neural Information
  Processing Systems}, 2018. [Online]. Available:
  \url{https://proceedings.neurips.cc/paper_files/paper/2018/file/bd1354624fbae3b2149878941c60df99-Paper.pdf}
\BIBentrySTDinterwordspacing

\bibitem{wang2023binary}
\BIBentryALTinterwordspacing
W.~Wang, L.~Feng, Y.~Jiang, G.~Niu, M.-L. Zhang, and M.~Sugiyama, ``Binary
  classification with confidence difference,'' in \emph{37th Conference on
  Neural Information Processing Systems}, 2023. [Online]. Available:
  \url{https://openreview.net/forum?id=4RoD1o7yq6}
\BIBentrySTDinterwordspacing

\bibitem{bhatia2000better}
R.~Bhatia and C.~Davis, ``A better bound on the variance,'' \emph{The American
  Mathematical Monthly}, vol. 107, no.~4, pp. 353--357, 2000.

\bibitem{continuous_mapping}
H.~B. Mann and A.~Wald, ``On stochastic limit and order relationships,''
  \emph{The Annals of Mathematical Statistics}, vol.~14, no.~3, pp. 217--226,
  1943.

\bibitem{ferraty2004nonparametric}
F.~Ferraty and P.~Vieu, ``Nonparametric models for functional data, with
  application in regression, time series prediction and curve discrimination,''
  \emph{Nonparametric Statistics}, vol.~16, no. 1-2, pp. 111--125, 2004.

\end{thebibliography}

\end{document}